%% file: main.tex
\documentclass{article}

\usepackage{PRIMEarxiv}

\usepackage[utf8]{inputenc} 
\usepackage[T1]{fontenc}    
\usepackage{hyperref}       
\usepackage{url}            
\usepackage{booktabs}       
\usepackage{amsfonts}       
\usepackage{nicefrac}       
\usepackage{algorithm}
\usepackage{algorithmic}
\usepackage{cite}
\usepackage{amsmath,amssymb,amsfonts}
\usepackage{algorithmic}
\usepackage{graphicx}
\usepackage{textcomp}
\usepackage{xcolor}
\usepackage{amssymb}

\usepackage{clrscode}
\usepackage{multirow}
\usepackage{multicol}
\usepackage{mathtools}
\usepackage{amsfonts}
\usepackage{array}

\usepackage{amsthm}
\usepackage{makecell}
\usepackage{tikz}
\usepackage{pgfplots}
\usepackage{booktabs}
\usepackage{tabularx}
\usepackage{caption}
\usepackage{subcaption}
\usepackage{enumitem}
\usepackage{xspace}

\usepackage{amsthm}
\usepackage{thmtools, thm-restate}
\usepackage{color,soul}
\usepackage{tabularray}

\newcommand{\mbbE}{\mathbb{E}}
\newcommand{\mbbR}{\mathbb{R}}
\newcommand{\mbbI}{\mathbb{I}}

\newcommand{\calR}{\mathbb{R}}

\newcommand{\calX}{\mathcal{X}}
\newcommand{\calY}{\mathcal{Y}}
\newcommand{\calS}{\mathcal{S}}
\newcommand{\calA}{\mathcal{A}}

\newcommand{\calD}{\mathcal{D}}
\newcommand{\calL}{\mathcal{L}}
\newcommand{\calM}{\mathcal{M}}
\newcommand{\calN}{\mathcal{N}}

\newcommand{\calO}{\mathcal{O}}
\newcommand{\calQ}{\mathcal{Q}}
\newcommand{\calW}{\mathcal{W}}

\newcommand*{\renyi}{R{\'e}nyi\@\xspace}

\usepackage{microtype}      
\usepackage{lipsum}
\usepackage{fancyhdr}       
\usepackage{graphicx}       
\graphicspath{{media/}}     

\title{RQP-SGD: Differential Private Machine Learning through Noisy SGD and Randomized Quantization
\thanks{This work is accepted by the 5th AAAI Workshop on Privacy-Preserving Artificial Intelligence.} 
}

\author{
  Ce Feng, Parv Venkitasubramaniam \\
  Department of Electrical and Computer Engineering \\
  Lehigh University \\
  City\\
  \texttt{\{cef419, pav309\}@lehigh.edu}
}

\begin{document}
\maketitle

\begin{abstract}
The rise of IoT devices has prompted the demand for deploying machine learning at-the-edge with real-time, efficient, and secure data processing. In this context, implementing machine learning (ML) models with real-valued weight parameters can prove to be impractical particularly for large models, and there is a need to train models with quantized discrete weights. At the same time, these low-dimensional models also need to preserve privacy of the underlying dataset. In this work, we present RQP-SGD, a new approach for privacy-preserving quantization to train machine learning models for low-memory ML-at-the-edge. This approach combines differentially private stochastic gradient descent (DP-SGD) with randomized quantization, providing a measurable privacy guarantee in machine learning. In particular, we study the utility convergence of implementing RQP-SGD on ML tasks with convex objectives and quantization constraints and demonstrate its efficacy over deterministic quantization. Through experiments conducted on two datasets, we show the practical effectiveness of RQP-SGD.
\end{abstract}

\input{paper/introduction}
\input{paper/related_work}
\input{paper/preliminary}
\input{paper/method_revision}

\input{paper/experiments}
\input{paper/conclusion}
\section{Acknowledgments}
This research is funded in part by Lehigh CORE grant CNV-S00009854, and the CIF 1617889 grant from the National Science Foundation.


\bibliographystyle{unsrt}  
\bibliography{references}

\end{document}

%% file: paper/introduction.tex
\section{Introduction}\label{sec:introduction}

As IoT devices proliferate across industries, there is an increasing demand to process data closer to the source~\cite{feng2022inferential,wang2022privstream}, enabling real-time insights and decision-making. There is now a growing need for machine learning (ML) at-the-edge, where ML models are deployed directly onto IoT devices or gateways, enabling them to perform data analysis and make predictions locally, without the need to transmit raw data to centralized servers. This approach not only reduces latency and conserves network resources but can also enhance privacy and security.

Although ML algorithms have shown tremendous success in various domains, ML at the edge brings unique constraints in the limited dimensionality of the models learned, and the need for strong privacy guarantees, particularly for IoT deployed in sensitive applications such as health monitoring and energy management. The objective of this work is to propose a joint privacy-preserving quantization approach to train neural networks for ML-at-the-edge IoT applications.


Past studies~\cite{yu2019differentially,yu2021large,papernot2021tempered,andrew2021differentially} propose different approaches to guaranteeing privacy in ML, notably through the concept of differential privacy~\cite{dwork2006differential} - a widely accepted quantitative measure of privacy. Specifically, these methods rely on a \textit{noisy} training approach known as \textit{Differentially Private Stochastic Gradient Descent (DP-SGD)}. DP-SGD~\cite{abadi2016deep} directly perturbs the gradient at each descent update with random noise drawn from the Gaussian distribution, resulting in a significant impact on utility. Several recent papers propose noise reduction methods for the differentially private noise added to the reduced gradient. For instance, recent studies by \cite{yu2021do, yu2021large, nasr2020improving} have explored methods for applying DP-SGD to gradients in a reduced dimensional space. Another work~\cite{feng2023spectral} performs differentially private perturbations in the spectral domain and introduces filtering for noise reduction. 

In this work, we propose Randomized Quantization Projection-Stochastic Gradient Descent (RQP-SGD), a new approach to achieving differential privacy in ML when weights need to be discretized. Our work is motivated by the knowledge that quantization is a process of removing redundant information by converting model parameters from high-precision to low-precision representations. Since privacy also requires the removal of sensitive information, our approach exploits the synergies between these two processes. 
Furthermore, quantization effectively reduces the memory and computation costs of deploying ML models, thus catering to the specific application in consideration.
In this work, our approach is to develop and analyze a randomized quantization approach in the training of ML models that can provide measurable differential privacy.
A study~\cite{xiong2016randomized} introduced a randomized requantization technique, primarily used as a compression mechanism in IoT systems and sensor networks. This technique simultaneously achieves local differential privacy and compression. However, it also highlights a notable trade-off between privacy, compression, and utility, a consequence of the information loss inherent in quantization. Our research extends the paradigm of randomized quantization to the domain of machine learning, adapting it into a projected stochastic gradient descent (Proj-SGD). We further show that our proposed approach achieves a better utility-privacy trade-off than deterministic quantization through theoretical analysis and experimental validation.



 



\textbf{Main Contribution} The main contributions of the paper are as follows: 
\begin{itemize}
    \item We propose RQP-SGD, a randomized quantization projection based SGD in ML with a differential privacy guarantee. 
    \item We theoretically study the utility-privacy trade-off of RQP-SGD for ML with convex and bounded loss functions. 
    \item Through experiments on two classification datasets: MNIST and Breast Cancer Wisconsin (Diagnostic) dataset~\cite{misc_breast_cancer_wisconsin_(diagnostic)_17}, the latter being a dataset collected from IoT devices, we demonstrate that RQP-SGD can achieve better utility performance than implementing DP-SGD in machine learning with quantized parameters. Significantly, RQP-SGD achieves a 35\% higher accuracy on the Diagnostic dataset while maintaining a $(1.0,0)$-DP, thereby validating its compatibility and efficacy for deployment in IoT systems.
    \item We conduct a comprehensive experimental analysis of how various RQP parameters influence the utility-privacy balance. This includes examining the effects of quantization-induced randomness, noise scale in gradient updates, and quantization bit granularity. Our findings highlight that while quantization-induced randomness can enhance utility, excessive randomness may have a detrimental effect on utility.
\end{itemize}


%% file: paper/related_work.tex
\subsection{Related Work}\label{sec:related_work}
The subject of differential privacy in ML has attracted significant scientific interest in recent years. Specifically, it has been used in support vector machine~\cite{li2014privacy}, linear/logistic regression~\cite{zhang2012functional, chaudhuri2008privacy} and risk minimization~\cite{bassily2014private, chaudhuri2011differentially, kasiviswanathan2016efficient, bassily2019private}. DP-SGD~\cite{abadi2016deep, yu2019differentially} as mentioned earlier perturbs the gradient at each SGD update. 

Quantization in ML is categorized into two types: post-training quantization (PTQ)~\cite{krishnan2019quantized, nagel2020up, nahshan2021loss}, and quantization aware training (QAT)~\cite{nagel2022overcoming,sakr2022optimal,courbariaux2015binaryconnect}. Our work fits into the latter category, modeling the quantized ML as a quantization-constrained optimization problem~\cite{bai2018proxquant}. 

Differentially private quantization scheme in releasing data has been studied in \cite{zhang2020optimal, xiong2016randomized}. We note another application of differential privacy and quantization in ML is federated learning. In recent years, there's been a growing body of research~\cite{lang2023joint, amiri2021compressive, gandikota2021vqsgd} that aims to study privacy-communication trade-off in federated learning. The main objective of private and efficient federated learning is transferring private, low-bandwidth gradient vectors to the server. 

%% file: paper/preliminary.tex
\section{Preliminary}\label{sec:preliminary}
\paragraph{Differential Privacy (DP)}\label{subsec:dp}
Differential privacy~\cite{dwork2006differential} is a quantitative definition of privacy, initially designed in the context of databases. Specifically, it ensures that whether or not an individual's data is included in a dataset does not significantly affect the analysis results on that dataset. 
A randomized mechanism $\calM:\calD\mapsto\calR$ satisfies $(\epsilon, \delta)$-DP if for any two adjacent sets $d,d'\in\mathcal{D}$ and all possible outputs $\calO$ of $\calM$, it holds that
$$\text{Pr}[\mathcal{M}(d)\in \calO]\leqslant e^{\epsilon} \text{Pr}[\mathcal{M}(d')\in \calO]+\delta$$
A prevalent approach in applying differential privacy to a real-valued function $f:\calD\mapsto\calR$ involves the addition of noise that is carefully calibrated to function's sensitivity $S_f$. The sensitivity, $S_f$, is defined as the maximum possible difference $|f(d)-f(d')|$ between the outputs of $f$ for any two adjacent inputs $d$ and $d'$. A common example is the Gaussian mechanism, which involves adding noise perturbed by a Gaussian distribution directly to the output of the function $f$. This process can be formulated as:
$$Gauss(f,d,\sigma)=f(d)+\calN(0,S_f^2\sigma^2)$$
where $\calN(0,S_f^2\sigma^2)$ is the Gaussian distribution with zero mean and variance $S_f^2\sigma^2$. The Gaussian mechanism achieves $(\epsilon,\delta)$-DP when $\sigma=\sqrt{2\log(1.25/\delta)}/\epsilon$~\cite{dwork2014algorithmic}.

\paragraph{Machine Learning As Empirical Risk Minimization}
In this paper, we frame the training of a machine learning (ML) model as empirical risk minimization (ERM), utilizing a dataset $\calS=\{(x_i, y_i)\in\calX\times\calY):i=1,2,\cdots,n\}$ of $n$ data-label pairs. The ML model is denoted by a predictor $f:\calX\times\calW\mapsto\calY$ featured by a set of parameters $w\in\calW$. The quality of the predictor on training data is quantified through a non-negative loss function $l:\calY\times\calY\mapsto\calR$. We aim to choose optimal $w$ that minimizes the empirical loss:
\begin{equation}\label{eq:ERM}
    \mathop{\min\limits_{w\in\calW}} \hat{\calL}(w;\calS):=\frac{1}{n}\sum_{i=1}^{n}l(f(x_i;w),y_i)
\end{equation}
Stochastic Gradient Descent (SGD) is a widely used optimization method in machine learning. At each iteration, SGD selects a mini-batch consisting of $m$ training samples and computes the stochastic gradient $\sum^m_{j=1}\nabla l(f(w_t;x_j);y_j)$. This gradient serves as an estimation of the gradient derived from the entire training dataset. This technique incrementally adjusts model parameters by the stochastic gradient, effectively guiding the model towards optimal performance.

\paragraph{Differential Privacy in Machine Learning}
In the realm of machine learning, the ability of models to discern intricate patterns from training datasets brings forth significant privacy concerns, particularly regarding the inadvertent memorization and subsequent exposure of individual data points. This issue becomes more pronounced with the advent of sophisticated techniques such as model inversion and membership inference attacks, which can exploit these vulnerabilities to infringe upon individual privacy.

In addressing the privacy concerns inherent in machine learning, the framework of differential privacy (DP) has emerged as a formalized methodology. It is intricately designed to quantify and attenuate the risks associated with the dissemination of information extracted from sensitive datasets. It ensures that machine learning model outputs are carefully calibrated to prevent the inference of sensitive information about any particular individual. 

Mathematically, for the machine learning predictor denoted as $f$, it satisfies $(\epsilon, \delta)$-DP if for any two adjacent sets $x, x'\in\calX$, the following holds:
\begin{equation}\label{eq:ML-dp}
    \text{Pr}[f(w; x)]\leq e^{\epsilon}\text{Pr}[f(w; x')] +\delta  
\end{equation}
This formula is foundational to the principle of differential privacy in machine learning. It is designed to control the probability distribution of outcomes from the private machine learning model in a manner that is minimally influenced by the presence or absence of any individual data within the training dataset.

A prevalent technique to attain differential privacy in machine learning is differentially private stochastic gradient descent (DP-SGD). DP-SGD modifies the standard SGD update rule to incorporate the Gaussian mechanism, as described by the following formulation:
\begin{equation}\label{eqq:dp-sgd}
     w_{t+1} = w_t-\eta \cdot\frac{1}{m}[\sum_{j=1}^m {\nabla} l(f(w_t;x_j);y_j) + G_t ]
\end{equation}
where $w_{t+1}$ denotes the updated model parameters, $\eta$ is the step size, and the gradient ${\nabla} l(f(w_t;x_j);y_j)$ is calculated at each $t$-th iteration using the data pair $(x_j, y_j)$. The key to DP-SGD is to add Gaussian noise vector $G_t$ which is calibrated to the sensitivity of ${\nabla} l(f(w_t;x_j);y_j)$. The inclusion of this noisy gradient, processed through the Gaussian mechanism, ensures the privatization of the gradient. The gradient descent serves as the post-processing, which inherently preserves the differential privacy.  




However, these techniques typically presume that the model weights are real-valued. Our research pivots from this norm by aiming to achieve DP with quantized weights, essential in resource-constrained environments where model size and computational efficiency are critical. This pursuit addresses the need for privacy-preserving models in environments where resources are limited and hence quantization is essential for reducing computational demands and model size.


\paragraph{Quantized ML optimization} In the context of quantized machine learning models, the parameters are restricted to a discrete set within the parameter space. This approach to training quantized ML models formulates the optimization problem as follows:
\begin{equation}\label{eq:Q-ERM}
    \mathop{\min}\limits_{w\in\calW}\hat{L}(w;\calS) 
    \begin{array}{l}
     \text{  s.t. } w\in\calQ     
    \end{array}
\end{equation}
where $\calQ\subseteq\mbbR^d$ represents a discrete, non-convex quantization set. Given that (\ref{eq:Q-ERM}) is an integer optimization with non-linear constraints, it necessitates relaxation to a form amenable to solution via projected SGD (Proj-SGD)~\cite{yin2018binaryrelax, li2017training}:
\begin{equation}\label{eq:PSGD}
    \left\{
    \begin{array}{l}
     v_{t+1} = w_t-\eta \nabla l(f(w_t;\cdot);\cdot)   \\
     \\
    w_{t+1} = \text{Proj}_{\calQ}(v_{t+1}) 
    \end{array}\right.
\end{equation}
where $\nabla l(f(w_t;\cdot);\cdot)$ is the sampled mini-batch gradient at the $t$-th iteration, $\eta$ is the step size, and $\text{Proj}_{\calQ}(v)=\mathop{\arg\min}_{u\in\calQ}\|u-v\|$ is a projection that projects $v$ onto the quantization set. 

%% file: paper/method_revision.tex
\section{Method}\label{sec:method}
\paragraph{Problem Setting}\label{subsec:problem_setting}
Our goal is to develop a differentially private optimization solution to (\ref{eq:Q-ERM}). More specifically, we aim to solve the following ERM problem:
\begin{equation}\label{eq:DP-Q-ERM}
\begin{array}{r}
     \mathop{\min}\limits_{w\in\calW}\hat{L}(w;\calS):=\frac{1}{n}\sum_{i=1}^{n}l(f(x_i;w),y_i)\\ 
     \\
  \text{  s.t. } \left\{\begin{array}{l}
    w\in\calQ  \\
    \\
   \text{Pr}[{f(w;x)}]\leq e^{\epsilon}\text{Pr}[{f(w;x')}]+\delta 
  \end{array}  \right.   
\end{array}
\end{equation}
where $x,x'$ are two adjacent subset of $\calX$. In this work, we incorporate two key assumptions:
\begin{itemize}
    \item The parameter space $\calW\subset\mbbR^d$ is a closed, convex set bounded by $M$: $\|w\|\leq M$, and the quantization set $\calQ$ is a discrete subset of $\calW$.
    \item For all data-label pair $(x_i, y_i)\in\calS$, the loss function, $l(f(x_i;w),y_i)$, is a convex and $\rho$-Lipschitz with respect to $w$, for example, binary cross-entropy loss in Logistic regression, and hinge loss in SVM classification.
\end{itemize}
A conventional approach to solving (\ref{eq:DP-Q-ERM}) is adapting DP-SGD (\ref{eqq:dp-sgd}) to the Proj-SGD (\ref{eq:PSGD}):
\begin{equation}\label{eq:DP-PSGD}
    \left\{
    \begin{array}{l}
     v_{t+1} = w_t-\eta \cdot\frac{1}{m}[\sum_{j=1}^m\nabla l(f(w_t;x_j);y_j) + G_t ] \\
     \\
    w_{t+1} = \text{Proj}_{\calQ}(v_{t+1}) 
    \end{array}\right.
\end{equation}
Here, $G_t\sim\calN(0,\sigma^2\mbbI_d))$ represents noise independently drawn at each SGD update. This adaptation of DP-SGD to Proj-SGD presents certain limitations: First, the deterministic nature of the projection step, although serving as a post-processing phase of DP-SGD, does not contribute additional privacy safeguards. Instead, it induces a "projection error", a consequence of aligning model parameters with the nearest point in $\calQ$.
Second, a stricter privacy budget necessitates scaling up the noise, which in turn leads to an increased noise error. 

To address these limitations, we propose randomized projection (RP), a novel methodology that integrates stochastic elements into the projection phase. The crux of this method lies in leveraging the projection phase as a mechanism to bolster privacy protection. By injecting controlled randomness into the projection, we target achieving the designated privacy budget while concurrently lowering the noise error. 

 
\paragraph{Randomized Projection} 
In randomized projection, we consider $b$-bit quantized parameters with uniformly distributed levels
$$\calQ_{M,b}=\{Q_0, Q_1,\cdots,Q_{2^b-1}\}$$
where $\calQ_{M,b}$ denotes the $b$-bit uniform quantization set with quantization bound $M\in\mbbR^+$, and each quantization level is given by
\begin{equation}\label{eq:UniformQuantizer}
    Q_i = -M + \frac{2M}{2^{b}-1}\cdot i
\end{equation}
Randomized projection is a variant of the classical projection of unquantized inputs onto the quantization set. Formally, the classical projection of input $Q_{in}$ onto $\calQ_{M,b}$ is
\begin{equation}\label{eq:DQP}
    \text{Proj}^D_{\calQ_{M, b}}(Q_{in})=\mathop{\arg\min}\limits_{Q\in\calQ}\|\lfloor Q_{in}, M\rceil-Q\|_2
\end{equation}
where $\lfloor \cdot, M\rceil$ denotes clipping function that clips the parameters into $[-M, M]$. In contrast to classical projection, randomized projection (outlined in Algorithm \ref{alg:RQP}) adds randomness by using a coefficient $q\in(0,1)$, enhancing privacy by making the input value less deducible from its deterministic projection.
\begin{algorithm}[h]
    \caption{{\text{Randomized Projection}}}
    \label{alg:RQP}
    \begin{algorithmic}[1]
        \REQUIRE $k$-bit uniform quantization set $\calQ_{M,b}$, randomness coefficients $q\in[\frac{1}{2^b-1},1)$, quantization input $Q_{in}$
        \STATE Find $Q^*=\text{Proj}^D_{\calQ_{M, b}}(Q_{in})$ using (\ref{eq:DQP})
        \STATE Project $Q_{in}$ onto $\calQ_{M,b}$ randomly:             
        $$\text{Proj}^R_{\calQ_{M,b}}(Q_{in})     
        = \left\{\begin{array}{cc}
           Q^*  &  \text{with probability }q\\
           Q  & \text{for } Q\in\calQ^u\setminus \{Q^*\}\ \text{with probability }\frac{1-q}{2^b-1}
        \end{array}\right.$$
        \STATE \textbf{Return} Output of $\text{Proj}^R_{\calQ_{M,b}}(Q_{in})$
    \end{algorithmic}
\end{algorithm}
\subsection{RQP-SGD}
Based upon randomized projection, we propose an iterative SGD method for solving (\ref{eq:DP-Q-ERM}), termed RQP-SGD. This method is comprehensively detailed in Algorithm \ref{alg:ProbProjNSGD}. RQP-SGD is an adaptation of DP-SGD into a variant of Proj-SGD that incorporates the randomized projection in place of the deterministic one.  

\begin{algorithm}[h]
    \caption{{\text{RQP-SGD}}}
    \label{alg:ProbProjNSGD}
    \begin{algorithmic}[1]
        \REQUIRE Training dataset: $\calS=\{(x_i, y_i)\in\calX\times\calY):i=1,2,\cdots,n\}$, $\rho$-Lipschitz, convex loss function $l$, convex set $\calW\subseteq\mbbR^d$, step size $\eta$, mini-batch size $m$, number of iterations $T$, Quantization set $\calQ_{M,b}$, projection randomness coefficient $q$.
        \STATE Choose arbitrary initial point $w_0\in\calW$.
        \FOR{$t=0$ to $T-1$}
        \STATE Sample a batch $B_t=\{(x_j, y_j)\}_{j=1}^m\leftarrow\calS$ uniformly with replacement.
        \STATE $v_{t+1}=w_t-\eta\cdot\frac{1}{m}[\sum^m_{j=1}\nabla l(f(w_t;x_j);y_j)+G_t)]$\\
        where $G_t\sim\calN(0,\sigma^2\mbbI_d)$ drawn independently each iteration.
        \STATE $w_{t+1}:= \text{Proj}^R_{\calQ_{M,b}}\left(v_{t+1}\right)$ where $\text{Proj}^R_{\calQ_{M,b}}$ denotes the random projection onto $\calQ_r$.
        \ENDFOR
        \STATE \textbf{Return} $W_T$.
    \end{algorithmic}
\end{algorithm}
\paragraph{Privacy of RQP-SGD}
The essence of RQP-SGD lies in its strategy for privacy enhancement. In step 4 of Algorithm \ref{alg:ProbProjNSGD}, it upholds DP by incorporating a Gaussian noise vector, which is integral to DP-SGD. Subsequently, step 5 amplifies privacy protection via randomized projections. This dual-pronged method, which effectively combines DP-SGD and randomized projection, is designed to achieve the desired privacy budget. It simultaneously aims to mitigate the adverse effects of noise on the model's accuracy and performance, thereby addressing the inherent limitations of escalated noise levels in traditional DP-SGD applications. The differential privacy guarantee of RQP-SGD is rigorously established in Theorem \ref{thm:privacy_probnsgd}.


\begin{restatable}{theorem}{Privacythm}
\label{thm:privacy_probnsgd}
    For any $\epsilon>0$, there exists mini-batch sampling rate $\frac{m}{n}$, training iterations $T$, quantization bit $b$ and randomness coefficient $p$, noise scale $\sigma$ such that Algorithm \ref{alg:ProbProjNSGD} achieves $(\epsilon, 0)$-DP.
\end{restatable}
\begin{proof}
    The key to the proof is determining the differential privacy budget for each update in the RQP-SGD process. This total budget is then assembled using a composition method. In each update of RQP-SGD, the gradient is calculated from a mini-batch dataset $B_t=\{(x_j, y_j)\}_{j=1}^m$, which is uniformly sampled from the training dataset $\calS$. The term $f_{sgd}(w_t;B_t)$ represents the standard SGD update employing the stochastic mini-batch $B_t$, which is defined as
    $$f_{sgd}(w_t;B_t)= w_t-\eta\cdot\frac{1}{m}\sum^m_{j=1}\nabla l(f(w_t;x_j);y_j)$$
    In contrast to $f_{sgd}(w_t;B_t)$, step 4 of Algorithm \ref{alg:ProbProjNSGD}, denotes as $f_v(w_t;B_t)$, adds noise drawn from normal distribution. The noise is scaled accordingly, leading to the formula:
    $$f_v(w_t;B_t) = f_{sgd}(w_t;B_t) + \frac{\eta}{m}\cdot\calN(0,\sigma^2\mbbI_d)$$
    The mechanism $\calM^t_{RQP}(w_t,B_t)$ denotes the differential private mechanism used in each RQP-SGD update with the mini-batch $B_t$. For any quantization level $Q_i\in\calQ_{M,b}$, the probability can be expressed as
    $$\begin{array}{rl}
    & Pr\{\calM^t_{RQP}(w_t,B_t)=Q_i\}
       =  \int Pr\{\text{Proj}^R_{\calQ_{M,b}}(v_{t+1})=Q_i\}
        \cdot Pr\{f_v(w_t;B_t)=v_{t+1}\} dv_{t+1}
    \end{array}$$
    This equation combines the probabilities related to the projection $\text{Proj}^R_{\calQ_{M,b}}(v_{t+1})$ achieving a particular quantization level $Q_i$ with those of the function $f_v(w_t;B_t)$ being equal to $v_{t+1}$.
    \begin{itemize}
        \item \textbf{Randomized projection:}\\
        For the projection input $v_{t+1}$, let $i^*=\arg\min_i \|Q_i-v_{t+1}\|^2$, the probability of $\text{Proj}^R_{\calQ_{M,b}}(v_{t+1})$ equaling to $Q_{i^*}$ is given by
        $$ Pr\{\text{Proj}^R_{\calQ_{M,b}}(v_{t+1})=Q_{i^*}\}=q\cdot Pr\{Q_{i^*}^-\leq v_{t+1}<Q_{i^*}^+\}$$
        where $Q_{i^*}^-=Q_{i^*}-\frac{M}{2^b-1}$ and $Q_{i^*}^+=Q_{i^*}+\frac{M}{2^b-1}$. 
        \\For any $Q_j\in\calQ_{M,b}$ but not equal to $Q_{i^*}$, the probability is
        $$ Pr\{\text{Proj}^R_{\calQ_{M,b}}(v_{t+1})=Q_{j}\}=\frac{1-q}{2^b-1}\cdot Pr\{Q_{j}^-\leq v_{t+1}<Q_{j}^+\}$$
        where $Q_{j}^-=Q_{j}-\frac{M}{2^b-1}$ and $Q_{j}^+=Q_{j}+\frac{M}{2^b-1}$.\\
        \item \textbf{Distribution of $f_v(w_t;B_t)$:}\\
        The function $f_v(w_t;B_t)$ introduces noise to $f_{sgd}(w_t;B_t)$. This implies that $v_{t+1}$ adheres to a normal distribution $\calN(f_{sgd}(w_t;B_t), (\frac{\eta}{m}\sigma)^2)$.
    \end{itemize}
    Based on the probability distributions of randomized projection and $f_v(w_t;B_t)$, the probability distribution of $\calM^t_{RQP}(w_t,B_t)$ for any $Q_i\in\calQ_{M,b}$ can be formulated as follows:
    $$\begin{array}{rl}
    &Pr\{\calM^t_{RQP}(w_t,B_t)=Q_i\} \\
    &\\
    =& q\cdot[\Phi(\frac{Q_{i}^+-f_{sgd}(w_t;B_t)}{\sigma_l})-\Phi(\frac{Q_{i}^--f_{sgd}(w_t;B_t)}{\sigma_l})]+
     \frac{1-q}{2^b}\cdot[\Phi(\frac{Q_{i}^+-f_{sgd}(w_t;B_t)}{\sigma_l})-\Phi(\frac{Q_{i}^--f_{sgd}(w_t;B_t)}{\sigma_l})]\\
     &\\
    =& \frac{2^bq-1}{2^b-1}[\Phi(\frac{Q_{i}^+-f_{sgd}(w_t;B_t)}{\sigma_l})-\Phi(\frac{Q_{i}^--f_{sgd}(w_t;B_t)}{\sigma_l})]
    +\frac{1-q}{2^b-1}
    \end{array}$$
    where $Q_{i}^-=Q_{i}-\frac{M}{2^b-1}$, $Q_{i}^+=Q_{i}+\frac{M}{2^b-1}$, and $\sigma_l=\frac{\eta}{m}\sigma$.
    The next step involves finding the upper bound of $\log\frac{Pr\{\calM^t_{RQP}(w_t,B_t)=Q_i\}}{Pr\{\calM^t_{RQP}(w_t,B'_t)=Q_i\}}$ for any two adjacent sets $B_t, B'_t\in\calS$. The upper bound of this logarithmic ratio is synonymous with the definition of the $\infty$-th order of \renyi divergence~\cite{van2014renyi}. Specifically, for two probability distributions $P_1$ and $P_2$ defined over $\calR$, the $\infty$-th order of \renyi divergence~\cite{van2014renyi} is given as
    $$D_{\infty}(P_1\|P_2)=\log\sup\limits_{d\in\calD}\frac{P_1(d)}{P_2(d)}$$
    Utilizing this definition, the following relation can be established for the RQP-SGD update:
    \begin{equation}\label{eq:privacy_bound}
        \begin{array}{cl}
        &\sup\limits_{Q_i\in\calQ_{M,b}}\log\frac{Pr\{\calM^t_{RQP}(w_t,B_t)=Q_i\}}{Pr\{\calM^t_{RQP}(w_t,B'_t)=Q_i\}} \\
        &\\
        =& D_{\infty}(Pr\{\calM^t_{RQP}(w_t,B_t)=Q_i\}\| Pr\{\calM^t_{RQP}(w_t,B'_t)=Q_i\})\\
        &\\
    =&\max\{\log\frac{\frac{2^bq-1}{2^b-1}[\Phi_+(f(B_t))-\Phi_-(f(B_t))]+\frac{1-q}{2^b-1}}{\frac{2^bq-1}{2^b-1}[\Phi_+(f(B'_t))-\Phi_-(f(B'_t))]+\frac{1-q}{2^b-1}}\}\\
    &\\
    \text{where }& \Phi_+(f(B_t))=\Phi(\frac{Q_{i}^+-f_{sgd}(w_t;B_t)}{\sigma_l})\\
    & \Phi_-(f(B_t))=\Phi(\frac{Q_{i}^--f_{sgd}(w_t;B_t)}{\sigma_l})\\
    & \Phi_+(f(B'_t))=\Phi(\frac{Q_{i}^+-f_{sgd}(w_t;B'_t)}{\sigma_l})\\
    & \Phi_+(f(B'_t))=\Phi(\frac{Q_{i}^--f_{sgd}(w_t;B'_t)}{\sigma_l})
    \end{array}\end{equation}
    Given that \renyi divergence is quasi-convex~\cite{van2014renyi}, (\ref{eq:privacy_bound}) achieves its maximum at the extreme points. 
    
    The loss function $l(f(w_t;\cdot);\cdot)$ is $\rho$-Lipschitz with respect to $w_t$, according to Lemma 14.7 in \cite{shalev2014understanding}. This means for any data-label pair $(x_i,y_i)\in\calS$, the norm of the gradient is bounded by $\rho$: $\|\nabla l(f(w_t;x_i);y_i\|\leq \rho$. With the assumption that $w_t$ is bounded by $M$, for any two adjacent $B_t, B'_t\in\calS$, the maximum difference of $f_{sgd}(w_t;\cdot)$ is constrained by $C$: $\max\|f_{sgd}(w_t;B_t)-f_{sgd}(w_t;B_t)\|\leq C$ where $C=M-\eta\rho$. Incorporating the extreme value of $f_{sgd}(w_t;B_t)$ into (\ref{eq:privacy_bound}), the following inequality is obtained:
    $$\log\frac{Pr\{\calM^t_{RQP}(w_t,B_t)=Q_i\}}{Pr\{\calM^t_{RQP}(w_t,B'_t)=Q_i\}}\leq\epsilon_t$$
    where $\epsilon_t=\log\frac{\frac{2^bq-1}{2^b-1}[2\Phi(\frac{a_1}{\sigma_l})-1]+\frac{1-q}{2^b-1}}{\frac{2^bq-1}{2^b-1}[\Phi(\frac{a_2}{\sigma_l})-\Phi(\frac{a_3}{\sigma_l})]+\frac{1-q}{2^b-1}}$, $a_1=\frac{M}{2^b-1}$,$a_2=M+\frac{M}{2^b-1}+C$, and $a_3=M-\frac{M}{2^b-1}+C$. This shows that each RQP-SGD update is $(\epsilon_t,0)$-DP with respect to the stochastic mini-batch $B_t$.

    Considering that RQP-SGD performs a total of $T$ iterations, with each iteration involving the uniform sampling of a stochastic mini-batch with replacement, the total privacy budget can be composed. By leveraging the DP composition theorem~\cite{dwork2014algorithmic} and amplification of DP via subsampling~\cite{balle2018privacy}, the overall privacy budget for RQP-SGD is established as $(T\frac{m}{n}\epsilon_t, 0)$-DP.
\end{proof}

\paragraph{Utility of RQP-SGD}
In Theorem \ref{thm:prob_quant_utility}, we provide the utility guarantee of Algorithm \ref{alg:ProbProjNSGD} based on the convergence analysis of the stochastic oracle model (see \cite{li2017training, shalev2014understanding}) with convex loss. Compared to the utility analysis in \cite{shalev2014understanding}, the RQP-SGD leads to two additional errors: Quantization error ($E_Q$) and noise error ($E_N$). The former is predominantly influenced by the quantization bits ($b$) and the randomness coefficient $(q)$, where a smaller value of $q$ induces higher randomness in quantization, subsequently increasing $E_Q$. However, this increase under a fixed privacy budget can reduce the reliance on added noise of differential privacy, thereby potentially decreasing $E_N$. This presents a delicate trade-off in RQP-SGD: optimizing the value of $q$ becomes crucial in balancing the quantization and noise errors to achieve effective performance within the constraints of the given privacy budget.
\begin{restatable}{theorem}{Utilitythm}
\label{thm:prob_quant_utility}
    Let $\bar{W}_T=\frac{1}{T}\sum^T_{t=1}w_t$. Suppose the parameter set $\calW$ is convex and $M$-bounded, and the quantization set $\calQ_p$ is generated by a randomized quantizer with probability $q$ and $b$-bit. For any $\eta>0$, the excess empirical loss of $\calA_{\text{ProjNSGD}}$ satisfies
    \begin{equation}\label{eq:RQP-utility}
        \begin{array}{cl}
         & \mbbE\left[\hat{\calL}(\bar{w}_T;\calS)\right]-\mathop{\min}\limits_{w\in\calW}\hat{\calL}(w;\calS) 
        \leq  \frac{M^2}{2\eta T}+E_Q+\frac{\eta \rho^2}{2}+E_N
    \end{array} 
    \end{equation}
    where $E_Q=dM^2[\frac{q}{(2^b-1)^2}+\frac{2^{b+1}(2^{b+1}-1)}{3(2^b-1)^2}(1-q)]$ denotes the quantization error and $E_N=\eta \sigma^2 d$ denotes the noise error. 
\end{restatable}
\begin{proof}
    The key to the proof is determining the boundary of the randomized projection. The projection process $w_{t+1}=\text{Proj}^R_{\calQ_{M,b}}(v_{t+1})$ involves the randomized projection of $\tilde{w}_{t+1}$ onto $\calQ_{M,b}$. For the projection, we have:     
    $$\begin{array}{rl}
       \mathop{\mbbE}\limits_{w\in\calQ_{M,b}}\left[\|v_{t+1}-w_{t+1}\|^2\right]
      \leq   & d(\frac{M}{2^b-1})^2\cdot q +d\mathop{\sum}\limits_{i=1}^{2^b-1}(\frac{2M}{2^b-1}i)^2\frac{1-q}{2^b-1}\\
      &\\
       \leq & d M^2\left[\frac{q}{(2^b-1)^2}+\frac{2^{b+1}(2^{b+1}-1)}{3(2^b-1)^2}(1-q)]\right]
    \end{array}$$
     By the triangle inequality, for $u\in\calW$,
    $$\|w_{t+1}-u\|^2\leq\|w_{t+1}-v_{t+1}\|^2+\|v_{t+1}-u\|^2$$

    Let $ E_Q=dM^2\left[\frac{q}{(2^b-1)^2}+\frac{2^{b+1}(2^{b+1}-1)}{3(2^b-1)^2}(1-q)]\right]$, we can derive the following inequality from the previous principles:
    \begin{equation}\label{eq:inequality}
        \|v_{t+1}-u\|^2-\|w_{t+1}-u\|^2\geq- E_Q
    \end{equation}

    We next analyze the excess empirical loss under the assumption that the loss function is $\rho$-Lipschitz and convex over $\calW$. According to Lemma 14.7 in~\cite{shalev2014understanding}, for all $w\in\calW$ and gradient $\nabla\in\frac{\partial l}{\partial w}$, the norm of the gradient is bounded: $\|\nabla\|\leq \rho$.
    
    Let $w^*$ be the optimal parameter in the parameter space $\calW$, defined as $w^*=\mathop{\min}\limits_{w\in\calW}\hat{\calL}(w;\calS)$. Given that $w_{t+1}$ is the projection of $v_{t+1}$ and $w^*\in\calW$, the following inequality is obtained: 
    $$\begin{array}{cl}
         & \|w_{t}-w^*\|^2-\|w_{t+1}-w^*\|^2 \\
         &\\
        \geq & \|w_{t}-w^*\|^2-\|v_{t+1}-w^*\|^2-E_Q\\
        &\\
        \geq & 2\eta \langle w_t - w^*,\nabla \rangle - \eta^2\|\nabla \|^2-\eta^2\|G_t\|^2 - E_Q
    \end{array}$$
    Taking expectation of both sides, rearranging, and using the fact that $\mbbE[\|\nabla\|^2]\leq \rho^2$ and $\mbbE[\|G_t\|^2]= d\sigma^2$, we have:
    $$\begin{array}{cl}
       \langle w_t - w^*,\nabla_t\rangle\leq  & \frac{1}{2\eta}\mbbE[\|w_t-w^*\|^2-\|w_{t+1}-w^*\|^2] +\frac{\eta}{2}\rho^2+E_Q+\eta\sigma^2 d
    \end{array}$$
    Given the convexity of the loss function $l$, we can further derive the bound of $\mbbE\left[\hat{\calL}(\bar{w}_T;\calS)\right]-\mathop{\min}\limits_{w\in\calW}\hat{\calL}(w;\calS)$: 
    $$\begin{array}{cl}
         & \mbbE\left[\hat{\calL}(\bar{w}_T;\calS)\right]-\mathop{\min}\limits_{w\in\calW}\hat{\calL}(w;\calS) 
        \leq  \frac{M^2}{2\eta T}+E_Q+\frac{\eta \rho^2}{2}+\eta\sigma^2 d
    \end{array} $$
    This is achieved through the analytical methods used for SGD applied to Convex-Lipschitz-Bounded functions~\cite{shalev2014understanding}.
\end{proof}
Theorem \ref{thm:prob_quant_utility} shows the convergence order for $E_Q$ is $O(dM^2[\frac{q}{2^b-1}+(1-q)])$, highlighting the influence of the coefficient $q$ on the utility bound.
\paragraph{RQP vs Proj-DP-SGD}
 We employ the utility bound formally characterized in Theorem \ref{thm:prob_quant_utility} as a basis to compare RQP-SGD against the adaption of DP-SGD to P-SGD (Proj-DP-SGD), which serves as a baseline. This baseline treats privatization and quantization as distinct processes. Our numerical analysis, illustrated in Figure \ref{fig:trade-off-numerical}, assesses RQP-SGD under a strict $(\epsilon,0)$-DP setting, contrasting it with Proj-DP-SGD under a slightly more relaxed $(\epsilon,10^{-7})$-DP. Results in Figure \ref{subfig:trade-off-q95} reveal that RQP-SGD has lower utility bounds than Proj-DP-SGD at equivalent privacy levels. However, as shown in Figure \ref{subfig:trade-off-q90}, a lower projection randomness coefficient in RQP-SGD, which adds more randomness, can result in a higher utility bound, indicating a critical balance between utility bound and privacy in these SGD frameworks.
\begin{figure}[h]
     \centering
     \begin{subfigure}[b]{0.48\linewidth}
         \centering
         \includegraphics[width=\textwidth]{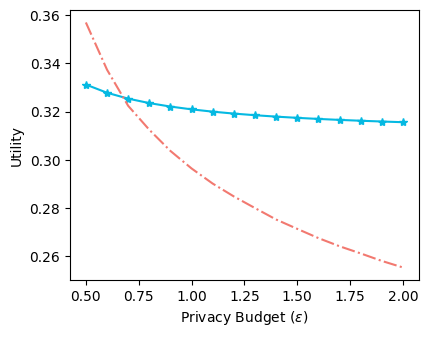}
         \vspace{-15pt}
         \caption{q = 0.90}
         \label{subfig:trade-off-q90}
     \end{subfigure}
    \begin{subfigure}[b]{0.48\linewidth}
         \centering
         \includegraphics[width=\textwidth]{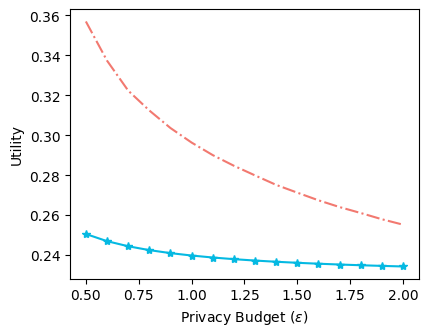}
         \vspace{-15pt}
         \caption{q = 0.95}
         \label{subfig:trade-off-q95}
     \end{subfigure}
     \caption{Noise scale ($\sigma$) and projection randomness coefficient ($q$) trade-off of RQP-SGD (blue) and Proj-DP-SGD. Parameter settings: $b=4,M=0.3,\rho=0.45,T=445,\frac{m}{n}=\frac{1}{445}$.}
     \label{fig:trade-off-numerical}
    \vspace{-10pt}
\end{figure}
\begin{figure*}[]
     \centering
     \begin{subfigure}[b]{0.23\linewidth}
         \centering
         \includegraphics[width=\textwidth]{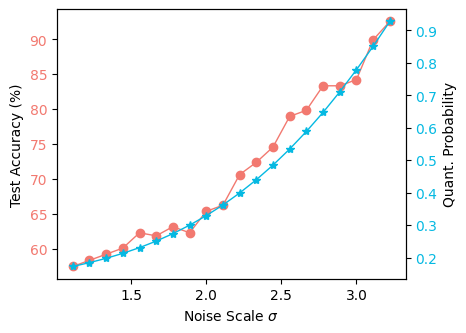}
         \vspace{-15pt}
         \caption{4-bit with $(0.5,0)$-DP}
         \label{subfig:b4e05}
     \end{subfigure}
    \begin{subfigure}[b]{0.23\linewidth}
         \centering
         \includegraphics[width=\textwidth]{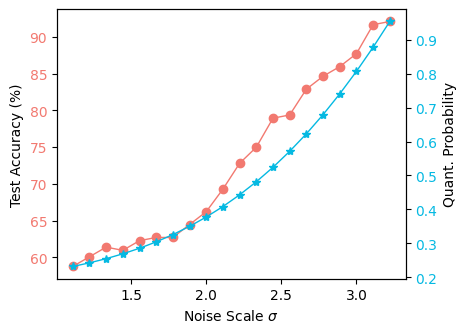}
         \vspace{-15pt}
         \caption{3-bit with $(0.5,0)$-DP}
         \label{subfig:b3e05}
     \end{subfigure}
     \begin{subfigure}[b]{0.23\linewidth}
         \centering
         \includegraphics[width=\textwidth]{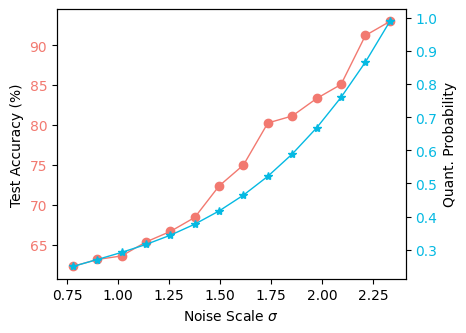}
         \vspace{-15pt}
         \caption{4-bit with $(1.0,0)$-DP}
         \label{subfig:b4e10}
     \end{subfigure}
     \begin{subfigure}[b]{0.23\linewidth}
         \centering
         \includegraphics[width=\textwidth]{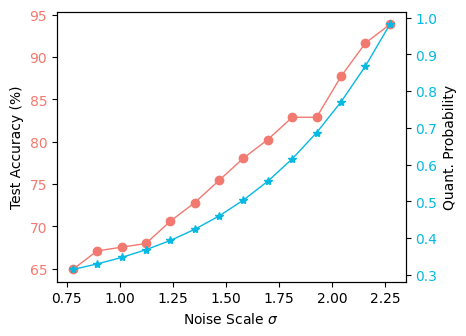}
         \vspace{-15pt}
         \caption{3-bit with $(1.0,0)$-DP}
         \label{subfig:b3e10}
     \end{subfigure}
     
     \caption{Noise scale ($\sigma$) and projection randomness coefficient ($q$) trade-off. Test accuracy values are the median over 10 runs.}
     \label{fig:privacy_accuracy}
    \vspace{-10pt}
\end{figure*}

%% file: paper/experiments.tex
\section{Experiments}\label{sec:experiments}
\paragraph{Setup}
We applied RQP-SGD to classification tasks using the Diagnostic~\cite{misc_breast_cancer_wisconsin_(diagnostic)_17} and MNIST datasets~\cite{lecun1998gradient}. The Diagnostic dataset, with 569 30-dimensional instances, was split into $80\%$ training and $20\%$ testing. MNIST contains 70,000 28x28 pixel grayscale images of handwritten digits, split into 60,000 for training and 10,000 for testing.
\paragraph{Training Details}We used logistic regression (LogRes)~\cite{wright1995logistic} and linear support vector machine (SVM)~\cite{hearst1998support} classifiers on the Diagnostic dataset, implementing DP-SGD, the adaption of DP-SGD to P-SGD (Proj-DP-SGD), and RQP-SGD with a mini-batch size of 10, step size of 1, and 46 training iterations. For the MNIST dataset, we implemented LogReg classifier with a mini-batch size of 64, step size of 1.0, and 938 training iterations. For all learning algorithms, we employed the gradient clipping technique with $l_2$ norm of 0.45. For RQP-SGD and Proj-DP-SGD, we set $[-0.3,0.3]$ as the parameter space bound and set $4$ as the quantization bit.

\paragraph{Results} We set ($1.0,10^{-7}$) as the privacy budget of DP-SGD and Proj-DP-SGD, and set $(1.0, 0)$ as the privacy budget of RQP-SGD. We report the test accuracy of the resulting ML models in Table \ref{tab:Log_results}. It clearly demonstrates that RQP-SGD achieves better utility performance than Proj-DP-SGD. Especially, the SVM classifier using RQP-SGD leads to 35.84\% median test accuracy gain among the DP-SGD with deterministic projection on the Diagnostic dataset. 

\begin{table}[h]
\centering
\caption{Test accuracy (\%) of Logistic regression classifiers with DP-SGD, Proj-DP-SGD, and $\calA_{\text{RQP-SGD}}$. We report median and standard deviation values over 10 runs. }
\label{tab:Log_results}
\resizebox{0.8\linewidth}{!}{
\begin{tabular}{cccc}
\hline
                             & \begin{tabular}[c]{@{}c@{}}Diagnostic\\ (LogReg Classifier)\end{tabular} & \begin{tabular}[c]{@{}c@{}}Diagnostic\\ (SVM Classifier)\end{tabular} & \begin{tabular}[c]{@{}c@{}}MNIST\\ (LogReg Classifier)\end{tabular} \\ \hline
\multicolumn{1}{c|}{Non-Private}  & 97.37\% (0.56\%)  &  98.68\% (0.68\%)  & 87.25\% (0.07\%)   \\
\multicolumn{1}{c|}{DP-SGD}  & 96.92\% (1.80\%)  & 96.49\% (1.24\%)   & 86.02\% (0.33\%)  \\
\multicolumn{1}{c|}{Proj-DP-SGD} & 94.30\% (1.41\%)  & 69.74\% (18.37\%)  & 84.32\% (0.29\%)  \\
\multicolumn{1}{c|}{RQP-SGD} & 95.18\% (1.42\%)  & 94.74\% (1.53\%)   & 84.81\% (0.30\%)  \\ \hline
\end{tabular}
}
\end{table}
\paragraph{Impact of Noise Scale and Projection Randomness}
To better understand the Privacy-Utility trade-off, we adjusted noise scales while maintaining a fixed privacy budget and quantization bits. The quantization randomness coefficient ($q$) is calculated by Theorem \ref{thm:privacy_probnsgd}. As presented in Fig. \ref{fig:privacy_accuracy}, the quantization randomness coefficient - noise scale curve (shown in the blue line) illustrates that decreasing quantization probability can enhance privacy while allowing less noise. 
On the utility front (represented by the red line), a lower projection randomness coefficient ($q$) leads to degraded test accuracy. 
For instance, when the projection randomness coefficient ($q$) is 0.3, the test accuracy decreases to around 62.5\%. The utility drop is attributed more to the randomness from quantization than from noise addition. 
From our observation, the standard deviation of test accuracy is higher when the projection randomness coefficient ($q$) is lower. This further illustrates the impact of the projection randomness coefficient ($q$) on utility.
\paragraph{Impact of Quantization Bit} 
We also extend our experiments with different quantization bits to explore the impact of the quantization bits. Based on our observation, the test accuracy does not increase as increasing the quantization bits.  This is because the higher the quantization bits, the less randomness is provided by quantization. To maintain the same privacy level, the randomness coefficient ($q$) is lower which results in higher utility loss. 

%% file: paper/conclusion.tex
\section{Conclusion}\label{sec:conclusion}
In this work, we propose RQP-SGD, a new approach to providing differential privacy in ML with quantized computational models. RQP-SGD combines differentially private noise addition with randomized quantization projection, which introduces additional randomness that enables the reduction of noise to improve utility. We theoretically analyze the feasibility of RQP-SGD in the training of ML models with convex objectives and validate the effectiveness of RQP-SGD through experiments on real datasets.

There are scopes for further research on RQP-SGD. First, the utility performance of RQP-SGD is highly sensitive to the projection randomness, RQP-SGD training with low quantization probability is still challenging, due to the high randomness. A comprehensive study of introducing randomness would be a promising avenue for future research. Second, there are other ML problems, such as non-convex optimization and training neural networks, that can be explored for further analysis. 